\documentclass[sigconf]{acmart}
\usepackage{multirow}
\usepackage{algorithm}
\usepackage[noend]{algcompatible}
\usepackage{graphicx}
\usepackage{subcaption}

\usepackage{balance}

\def\BibTeX{{\rm B\kern-.05em{\sc i\kern-.025em b}\kern-.08emT\kern-.1667em\lower.7ex\hbox{E}\kern-.125emX}}

\copyrightyear{2019} 
\acmYear{2019} 
\setcopyright{acmcopyright}
\acmConference[KDD '19]{The 25th ACM SIGKDD Conference on Knowledge Discovery and Data Mining}{August 4--8, 2019}{Anchorage, AK, USA}
\acmBooktitle{The 25th ACM SIGKDD Conference on Knowledge Discovery and Data Mining (KDD '19), August 4--8, 2019, Anchorage, AK, USA}
\acmPrice{15.00}
\acmDOI{10.1145/3292500.3330731}
\acmISBN{978-1-4503-6201-6/19/08}

\begin{document}

%
\title{Sample Adaptive Multiple Kernel Learning for Failure Prediction of Railway Points}

\author{Zhibin Li}
\affiliation{%
  \institution{University of Technology Sydney}
  \streetaddress{Broadway}
  \city{Sydney}
  \country{Australia}}\email{Zhibin.Li@student.uts.edu.au}

\author{Jian Zhang}
\affiliation{%
  \institution{University of Technology Sydney}
	\streetaddress{Broadway}
	\city{Sydney}
	\country{Australia}}\email{Jian.Zhang@uts.edu.au}

\author{Qiang Wu}
\affiliation{%
	\institution{University of Technology Sydney}
	\streetaddress{Broadway}
	\city{Sydney}
	\country{Australia}}\email{Qiang.Wu@uts.edu.au}

\author{Yongshun Gong}
\affiliation{%
	\institution{JD AI Research and \\University of Technology Sydney}
	\city{Beijing}
	\country{China}}\email{gongyongshun@jd.com}
\author{Jinfeng Yi}
\affiliation{%
	\institution{JD AI Research}
	\city{Beijing}
	\country{China}}\email{yijinfeng@jd.com}

\author{Christina Kirsch}
\affiliation{%
	\institution{Sydney Trains-Operational Technology}
	\city{Sydney}
	\country{Australia}}\email{Christina.Kirsch@transport.nsw.gov.au}
%
\renewcommand{\shortauthors}{Zhibin, et al.}

%
\begin{abstract}
Railway points are among the key components of railway infrastructure. As a part of signal equipment, points control the routes of trains at railway junctions, having a significant impact on the reliability, capacity, and punctuality of rail transport. Meanwhile, they are also one of the most fragile parts in railway systems. Points failures cause a large portion of railway incidents. Traditionally, maintenance of points is based on a fixed time interval or raised after the equipment failures. Instead, it would be of great value if we could forecast points' failures and take action beforehand, minimising any negative effect. To date, most of the existing prediction methods are either lab-based or relying on specially installed sensors which makes them infeasible for large-scale implementation. Besides, they often use data from only one source. We, therefore, explore a new way that integrates multi-source data which are ready to hand to fulfil this task. We conducted our case study based on Sydney Trains rail network which is an extensive network of passenger and freight railways. Unfortunately, the real-world data are usually incomplete due to various reasons, e.g., faults in the database, operational errors or transmission faults. Besides, railway points differ in their locations, types and some other properties, which means it is hard to use a unified model to predict their failures. Aiming at this challenging task, we firstly constructed a dataset from multiple sources and selected key features with the help of domain experts. In this paper, we formulate our prediction task as a multiple kernel learning problem with missing kernels. We present a robust multiple kernel learning algorithm for predicting points failures. Our model takes into account the missing pattern of data as well as the inherent variance on different sets of railway points. Extensive experiments demonstrate the superiority of our algorithm compared with other state-of-the-art methods. 
\end{abstract}

%
%


%

\maketitle

\section{Introduction}
Railway points are a kind of mechanical installations allowing railway trains to be guided from one track to another. They are among the key components
of railway infrastructure.

A railway junction is controlled jointly by one or more ends of points. They work together to control the routes of trains. In this paper, we use the term "a set of railway points" to indicate the entire mechanism in a railway junction. 

Apart from delay and cancellation of trains, failure of points can also cause severe economic loss and casualties. Railway points count for almost half of all train derailments in the UK \cite{ishak2016safety}. On the morning of 12 December, 1988, Clapham Junction rail crash \footnote{\url{https://en.wikipedia.org/wiki/Clapham_Junction_rail_crash}} killed 35 people, and injured 484 people. More than 20\% of incidents in Sydney Trains rail network were caused by points failures. Maintaining railway points safe, and forecasting the incoming failure are vital tasks for reliable rail transportation.

Routine maintenance is usually performed on railway points to ensure the correctness and reliability of them. Such work is done by field engineers to inspect and test the equipment at a fixed time interval. However, this strategy cannot catch the rapid change of equipment status. For example, when extreme weather occurs, points often degrade faster than usual. As a result, they are more likely to fail soon. Instead of relying on passive routine maintenance, we could benefit more from predictive maintenance - which flexibly arranges the maintenance work according to the running condition of equipment. 

Forecasting the failures is a critical step in predictive maintenance. Some research has been conducted on this topic \cite{camci2016comparison,garcia2010railway,oyebande2002condition,tao2015intelligent,yilboga2010failure}. Delicate sensors usually serve as data collectors for voltages, currents and forces in related work. Installation of sensors incurs costly labour and material expenses, as well as the possibility of sensor malfunction. Adding sensors for in-service equipment would also induce disruption to traffic. This is especially unacceptable for a large and busy rail network. These make the prediction with sensors' data expensive, or even infeasible. On the contrary, one can easily collect heterogeneous data from other sources such as weather, movement logs, and equipment details without an additional hardware upgrade.

Gathering available data from multiple sources enriches our knowledge on the working status of points. However, this also brings extra problems. Firstly, data collected from different sources are often in incompatible formats, and they play different roles in revealing the condition of equipment. Secondly, we are not guaranteed that data are always intact - even for a single source. Actually, in most case, we can only feed incomplete data into our model. Besides, our data were collected upon 350 sets of railway points. They are possibly located in a rural area, city centre, or from a different point of view, bridges, tunnels. They can also be of various types and made by different manufacturers. These add up to the difficulties in designing models. To summarise, we are faced with three main challenges here:
\begin{itemize}
\item How to combine information from multiple sources efficiently and effectively?
\item How to deal with missing data?
\item How to consider the distinct and shared properties between different sets of railway points simultaneously?                   
\end{itemize}

To address these challenges, we proposed a novel multiple kernel learning algorithms. Our method was developed based on the multiple kernel learning framework \cite{gonen2011multiple}. Multiple kernel learning has attracted much attention over the last decade. It has been regarded as a promising technique for combining multiple data channels or feature subsets \cite{xu2010simple}, which exactly meets our requirements.
We applied different kernel mapping functions on our data from different sources. Besides, we also concatenated all the data to form a kernel so that the inter-source correlations could be found. An adaptive kernel weight determined by both properties of an individual set of railway points and the missing pattern of data makes our model robust, effective and unique. The contributions of this paper can be shown in the following aspects:
\begin{itemize}
\item We provide a universal framework to predict points' failure with multi-source data. Our data are easy to obtain for most of the rail networks over the world without a hardware upgrade, and thus could be used in many other rail networks.  
\item Our work firstly introduces missing pattern adaptive kernel weight into existing multiple kernel learning framework. 
\item With a sample adaptive kernel weight, our model can capture the distinct and share properties of different railway points.
\item We developed an optimisation algorithm to optimise the proposed model. Through random feature approximation together with mini-batch gradient descent, the proposed method can be applied on large datasets. 
\item We conducted experiments on a real-world dataset collected from a wide range of railway points over three years. The results clearly show the effectiveness of our model.
\end{itemize}

The rest of this paper is organised as follows. Section 2 presents the related work. In Section 3, we describe our data and application. The proposed adaptive multiple kernel learning is detailed in Section 4. The experiment results are shown in Section 5. Last we conclude our work in Section 6.
\section{RELATED WORK}
We give a brief introduction to failure prediction of railway points and the multiple kernel learning (MKL) algorithm.
\subsection{Failure Prediction of Railway Points}
Knowing that railway points directly affect the capacity and reliability of rail transport, some research has been conducted on failure prediction of railway points \cite{camci2016comparison,garcia2010railway,oyebande2002condition,tao2015intelligent,yilboga2010failure}. Sensor data such as voltages, currents and forces were widely used in these works. They were collected in laboratories or from site sensors. These data would require a high sampling rate and lead to difficulties in both transmission and storage. Despite the success shown in these methods, they are impractical in real application.   

Few works explored the prediction task with data from another source. Weather plays a significant role in the probability of failure \cite{hassankiadeh2011failure}, and has been used to predict the total number of failed turnout systems in a railway network \cite{wang2017bayesian}. Note that this work could not locate the exact fault railway points, it only estimates the total number of failures in a large system. Apart from weather data, equipment logs are also valuable information for foreseeing the failures of related equipment \cite{sipos2014log}. Logs can be generated by sensors, software applications and even maintenance records \cite{li2018field}, reflecting the working condition of a piece of equipment in a different view. In \cite{li2018field}, maintenance logs are used to forecast the failure between two scheduled maintenance.

Many of above-mention methods used support vector machines (SVM) \cite{chang2011libsvm} for their models. They mainly focused on data from one source. A natural extension is to use multiple kernel learning to formulate our multi-source problem, and level up the performance. 

\begin{figure*}[t]
    \centering
    \includegraphics[width=\linewidth]{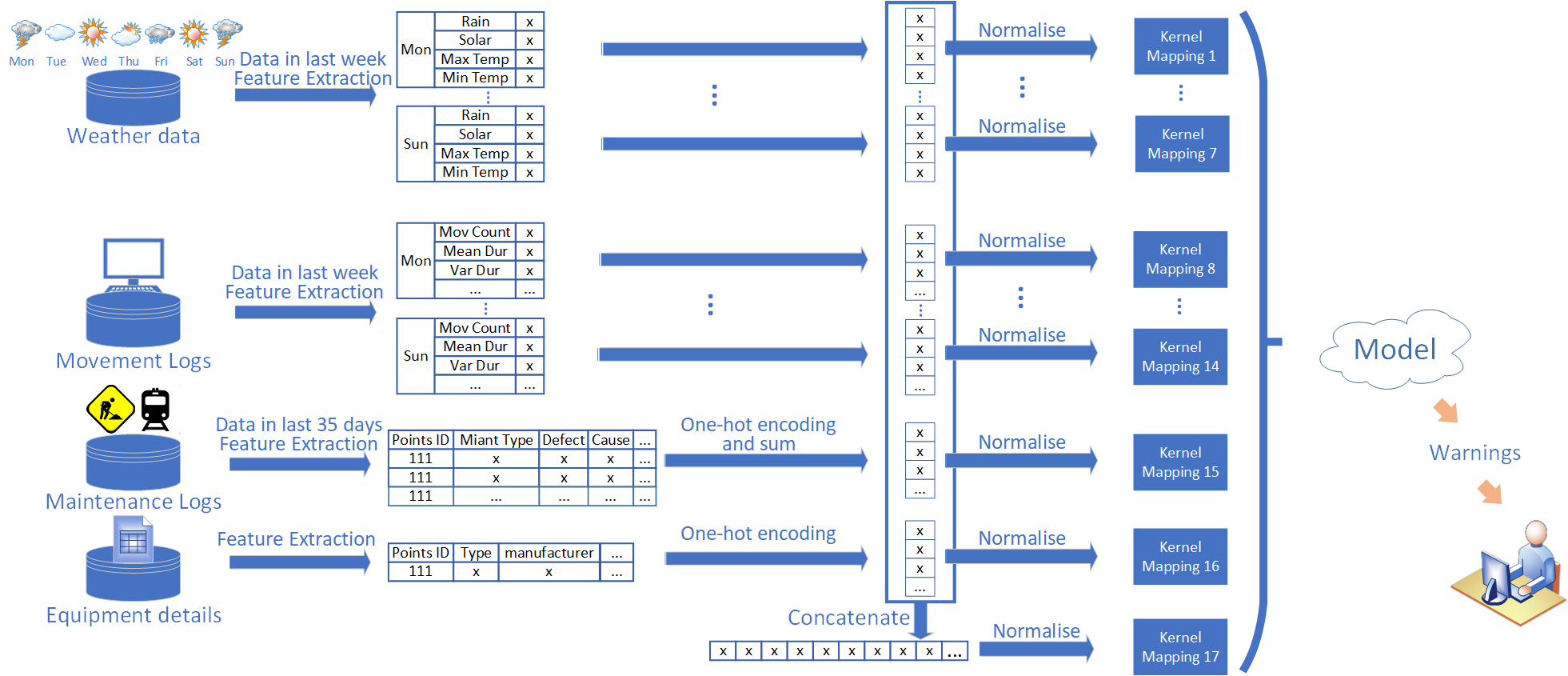}
    \caption{Workflow of our method.}\label{workflow}
\end{figure*}
\subsection{Multiple Kernel Learning}
Similar to deep neural networks, functions defined in reproducing kernel Hilbert space (RKHS) can model highly nonlinear relationship. MKL further takes the advantages of such functions by combining them wisely. Compared to deep neural networks, MKL enjoys better interpretability while requires less training data, which is more in line with our fundamental requirements.

MKL searches for an optimal combination of kernel functions to maximise a generalised performance measure. It has been widely used in various regression and classification tasks \cite{bucak2014multiple,althloothi2014human,yeh2011multiple,liu2014multiple,yang2012group}. 

For sample $\mathbf{x}_i=[\mathbf{x}^{(1)^ \top}_i,\mathbf{x}^{(2)^ \top}_i, \cdots ,\mathbf{x}^ {(s)^\top}_i]^\top$ consists of $s$ feature subsets, by applying $s$ mapping functions to each subset, it takes the form of:
\begin{equation}
\phi(\mathbf{x}_i) = 
[\phi_1^\top (\mathbf{x}_i^{(1)}),\phi_2^\top (\mathbf{x}_i^{(2)}), \cdots , \phi_s^\top(\mathbf{x}_i^{(s)})]^\top,
\end{equation}
where $\{\phi_m(\cdot) \}_{m=1}^s$ denote feature mappings associated with $m$ pre-defined base kernels $\{\kappa_m(\cdot,\cdot)\}_{m=1}^s$. Given samples $\left\{(\mathbf{x}_i, y_i)\right\}_{i=1}^n$ with $y_i \in \{-1,+1\}$ the label for $\mathbf{x}_i$, commonly used MKL can be formulated as the following convex optimisation problem \cite{rakotomamonjy2008simplemkl}:
\begin{equation}
\begin{aligned} 
\min_{\{\boldsymbol\omega_m\}_{m=1}^s,b,\boldsymbol{\xi},\boldsymbol{\eta}\in\Delta} \;& \frac12\sum_{m=1}^s\left\|\boldsymbol\omega_m\right\|_2^2 + C\sum_{i=1}^n\xi_i,\\ 
s.t.\quad& y_i\left(\sum_{m=1}^s\sqrt{\eta_m}\boldsymbol\omega_m^\top \phi_m(\mathbf{x}_i^{(m)}) + b\right) \geq 1-\xi_i, \\
& \xi_i \geq 0,\quad i=1,2,...,n,
\end{aligned} \label{eq2}
\end{equation}
where $\left\|\cdot\right\|_2$ is the Euclidean norm for vectors. $\boldsymbol\omega_m$ is the weight vectors for mapped features $\phi_m(\mathbf{x}_i^{(m)})$. $\boldsymbol\eta$ contains the weights for combination of base kernels. For $L_1$-norm of kernel weights, $\Delta=\{\boldsymbol\eta \in \mathbb{R}_+^s:\sum_{m=1}^s \eta_m=1, \eta_m\geq 0\}$. $b$ is the bias term and $C$ is a regularisation parameter for $\boldsymbol{\xi}$ which consists of slack variables. 
The decision score of the classifier on a sample $\mathbf{x}$ is given by:
\begin{equation}
f(\mathbf{x})=\sum_{m=1}^s\sqrt{\eta_m} \boldsymbol\omega_m^\top\phi_m(\mathbf{x}^{(m)}) + b.
\end{equation}

Many variants of the MKL have been proposed to improve the accuracy of MKL algorithms. A natural extension is to change the $L_1$-norm constraint for kernel weights to $L_p$-norm as in \cite{kloft2009efficient}. Algorithms in \cite{kloft2011lp} further simplified the optimisation procedure by adopting a closed-form solution for kernel weights. In \cite{liu2014sample}, a binary vector was introduced for every sample to switched on/off base kernels. The optimisation problem was an integer linear programming problem. The work in \cite{gonen2008localized} put forward a localised MKL algorithm. They utilised a gating model for selecting the appropriate kernel function locally. A convex variant was presented in \cite{lei2016localized} and corresponding generalisation error bounds were provided.

Another branch of studies focuses on improving the efficiency and scalability of MKL. In \cite{sonnenburg2006large}, they worked on a special scenario that when feature maps were sparse and can be explicitly computed. Combined with chunking optimisation, they were able to deal with large volumes of data. The work in \cite{rakotomamonjy2014lp} improved the scalability of MKL through Nystrom methods to approximate the kernel matrices and used proximal gradient algorithm in optimisation. Some research was also developed for the situation when the number of kernels to be combined was very large \cite{afkanpour2013randomized}. Besides, many online methods for MKL were proposed recently \cite{shen2018online,shen2018online2,li2017triply,sahoo2014online}. Random feature approximation \cite{rahimi2008random} is popular among these methods.

\begin{figure*}[t]
    \centering
    \begin{subfigure}[t]{0.45\textwidth}
        \centering
        \includegraphics[width=\linewidth]{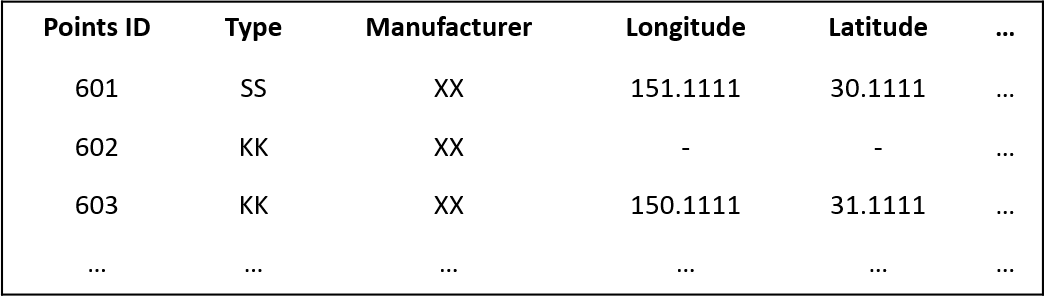}
        \caption{A piece of equipment details.}\label{eqtdtls}
    \end{subfigure}\quad\quad
    ~ 
    \begin{subfigure}[t]{0.45\textwidth}
        \centering
        \includegraphics[width=\linewidth]{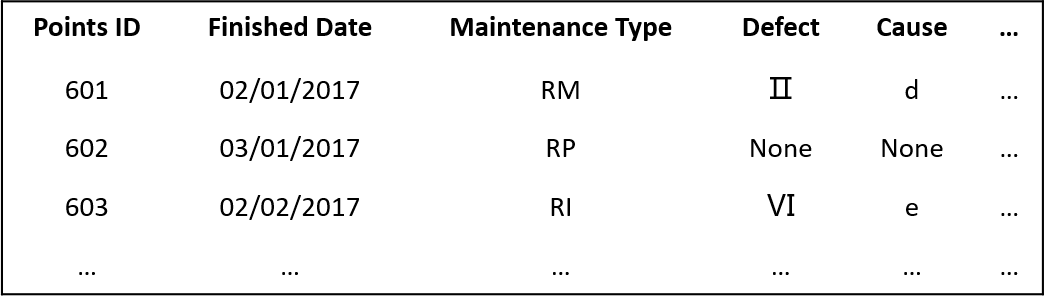}
        \caption{A piece of maintenance log.}\label{mlog}
    \end{subfigure}
    \\[2ex]
    \begin{subfigure}[t]{0.45\textwidth}
        \centering
        \includegraphics[width=\linewidth]{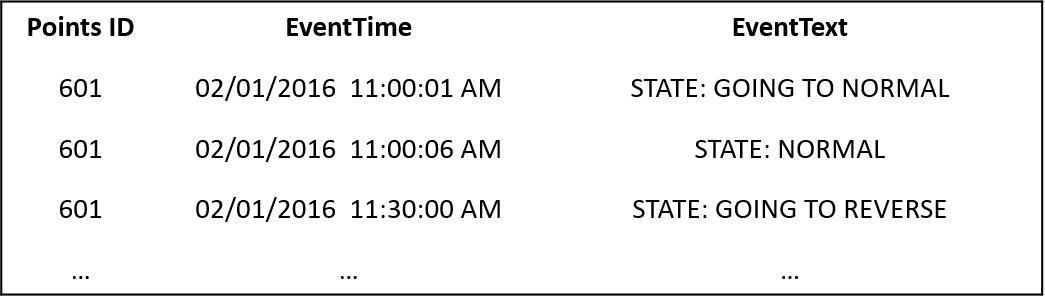}
        \caption{A piece of movement log.}\label{elog}
    \end{subfigure}\quad\quad
    ~ 
    \begin{subfigure}[t]{0.45\textwidth}
        \centering
        \includegraphics[width=\linewidth]{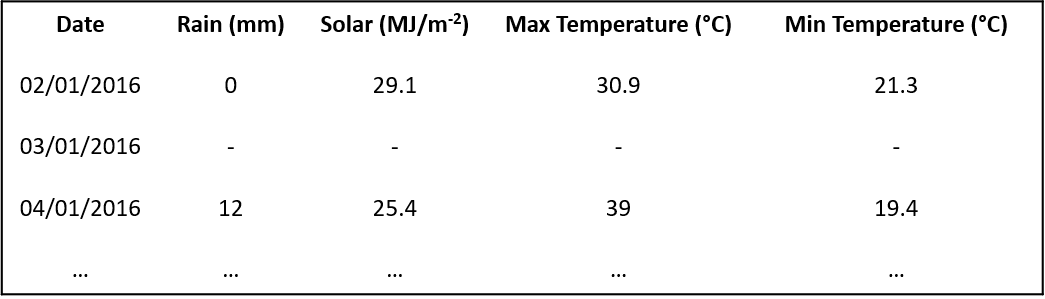}
        \caption{A piece of weather data.}\label{weather}
    \end{subfigure}
    \caption{A sample of our data.}
\end{figure*} 

Except for the work in \cite{liu2015absent}, most of the research on multiple kernel classification is based on the prerequisite that all kernels are complete, whereas in our problem, this is not true. The method in \cite{liu2015absent} cannot be scaled up to fit our dataset, and they actually treated different missing patterns equally in the test. We thus argue that this is insufficient. These inspire us to design a new algorithm that can handle a large dataset, and explore the benefits by not only dealing with different missing patterns accordingly but also treating each group of sample adaptively.

\section{Problem Description}
In this section, we describe our data and application. Figure \ref{workflow} shows the workflow of our method. 

\subsection{Data Description}
We collected railway points' equipment details, maintenance logs, movement logs and failure history from Sydney Trains database in a time range from 01/01/2014 to 30/06/2017. These data are collected from 350 sets of railway points spread in a large area. We also downloaded the weather data from Australia Bureau of Meteorology\footnote{\url{www.bom.gov.au/climate/data/}} of the same time span. Below we are going to introduce their formats and features.

\subsubsection{Infrastructure Failure Management System Database}
Infrastructure Failure Management System (IFMS) Database stores failures of assets in Sydney Trains with timestamps. We extracted points' failures as part of our ground truth. 

\subsubsection{Equipment Details}
Equipment details data record the detailed parameters of every set of railway points, including Points ID, Manufacturer, Type and so on. A piece of data is presented in Figure \ref{eqtdtls}. We use "-" to denote missing values. With the help of domain experts, we selected a subset of features from these columns, and they were all categorical variables. We would simply perform one-hot encoding with them. 

\subsubsection{Maintenance Logs}
Maintenance logs contain formatted historical maintenance logs of railway points. A subset of categorical features was extracted from them following advice by the domain experts. A piece of data is presented in Figure \ref{mlog}.

\subsubsection{Movement Logs}
Movement logs were automatically generated by Sydney Trains control system in a real-time manner. This system recorded states' changes of the railway points with timestamps in seconds. A piece of data is shown in Figure \ref{elog}. We only list some of the event types here. Failures are reported in logs as well. Some of the failures occurred in movement logs didn't appear in the IFMS database, for the reason that they recovered soon and didn't result in any significant incident. They were still real failures, and we included these failures in our ground truth. Sometimes workers were testing the points for preventative maintenance and this also generated failure logs. In this case, we ignore these failures to keep the ground truth clean.

\subsubsection{Weather}
Weather data were retrieved from the Australia Bureau of Meteorology. Our data were gathered from railway points spread in a large area, so weather conditions for them may vary. Our strategy was to download data from the nearest weather station according to the longitudes and latitudes provided by equipment details. Sometimes weather station would be closed for a while, and we were not able to find another station to substitute them in some situations. Some points are lack of geo-coordinates in Sydney Trains system. These cause the absence of weather data. Figure \ref{weather} shows a piece of weather data.

\subsection{Problem Formulation}
With data mentioned above in hand, we are going to make use of them to fulfil the prediction task. Essentially, this is a classification task. Since our data were generated from multiple sources, they came with different formats and sample frequencies. The two most important things are how we should aggregate our data from multiple sources and label them according to failure records.

Grouping and labelling data in a daily manner is an intuitive way. However, our data are highly imbalanced in label distribution. The number of days that failures occurred is about 4200, while our data include 454237 days summing over all railway points. This would produce a dataset contains only 0.9\% positive samples if we give a label "1" to failures. Such imbalanced dataset would deteriorate the performance of the classifier. 

Sydney Trains' train timetable shows cyclic patterns following calendar weeks \cite{gong2018network}, which will pose a periodic effect on our data as well. Therefore, we grouped our data according to calendar weeks. We gave label "1" to a week if any failure was recorded in IFMS or movement log of this week. As a result, our task is to predict whether there will be failures occur in any time of next week, depending on weather conditions, movement logs in this week and maintenance logs in a period of 35 days before next week. For maintenance logs, we extend the time range to 35 days since they were often performed based on a monthly interval. We would also incorporate equipment details, and in general, they are independent of time. Figure \ref{labelling} illustrates our data aggregation and labelling strategy. After some data cleaning, we finally generated 58833 samples including 3900 positive samples. 
\begin{figure}[t]
    \centering
    \includegraphics[width=\linewidth]{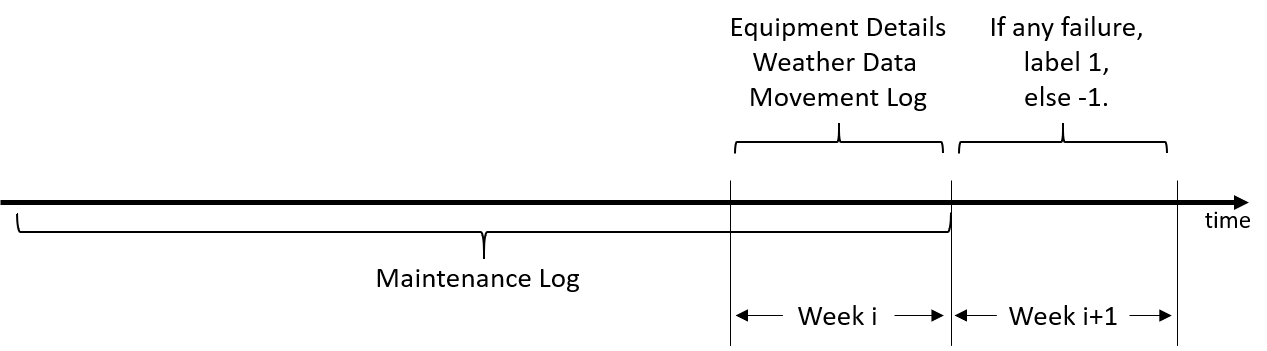}
    \caption{To forecast failures in week i+1, we use data from week i and maintenance logs in a 35-day interval before week i+1. }\label{labelling}
\end{figure}

Notice that in some cases we would lose the movement logs, for example, the influence of maintenance work. In these cases, we would only refer to logs in the IFMS database as failure indicators upon agreement with the domain experts.

\section{Methodology}
\subsection{Feature Extraction and Partition}
Although we have grouped our data according to the above-mentioned criterion, we need to flatten them further to form feature vectors. For equipment details and maintenance logs data, we selected some columns following the advice of domain experts. Then we performed one-hot encoding on these data. We summed up features if there are more than one maintenance records. For movement logs data, we extracted some statistical features for every day like mean of movements, variance of movements, count of movements and so on. Because there are 7 days per week, we would have 7 subsets of features for movement logs. Similarly, for weather data, we have 7 subsets for one week. This strategy could be seen in Figure \ref{workflow}. Such partition lets us easily handle the missing pattern in a daily format as we will introduce in detail in the next section. Table \ref{missing percentage} summarises missing percentages of our data after such feature partition. 

There are 16 feature subsets in total. By applying different kernel functions to different subsets, we can formulate our task as a multiple kernel learning problem for binary classification. In order to learn the interaction among feature subsets, we also concatenated all feature subsets to form a long vector and applied a kernel function on it. Finally, we would get 17 kernels as our inputs. We term these feature subsets \textbf{channels}.

The missing rates for each channel are not very high, but another fact is that 44\% of our data are either missing one channel or more. Therefore, it is imperative for us to build a model that is suitable for such data.

\subsection{Select Kernel Functions}
After applying one-hot encoding, features generated from equipment details and maintenance logs data were often very sparse. We thus directly used linear kernel for these two data channels as recommended in literature \cite{li2015data,fan2008liblinear}. For the remaining data channels consist of weather and movement logs of 7 days, we applied the commonly used radial basis function (RBF) kernels. In the rare case, some channels of a sample were only partially missing. If so, we filled the missing part with means.

\begin{table}[t]
    \centering
    \caption{Missing rates and dimensions of our data channels. 44\% of samples are missing at least one channel.}
    \label{missing percentage}
    \small
    \begin{tabular}{|c|c|c|c|}
        \hline
        \multicolumn{2}{|c|}{Data}                 & Missing Rate & Feature Dimension \\ \hline
        \multicolumn{2}{|c|}{Equipment Details}    & 0\%          & 450               \\ \hline
        \multicolumn{2}{|c|}{Maintenance Logs}     & 13\%         & 365               \\ \hline
        \multirow{7}{*}{Movement Logs} & Monday    & 5\%          & 30                \\ \cline{2-4} 
        & Tuesday   & 6\%          & 30                \\ \cline{2-4} 
        & Wednesday & 5\%          & 30                \\ \cline{2-4} 
        & Thursday  & 5\%          & 30                \\ \cline{2-4} 
        & Friday    & 7\%          & 30                \\ \cline{2-4} 
        & Saturday  & 8\%          & 30                \\ \cline{2-4} 
        & Sunday    & 10\%         & 30                \\ \hline
        \multirow{7}{*}{Weather}       & Monday    & 26\%         & 4                 \\ \cline{2-4} 
        & Tuesday   & 26\%         & 4                 \\ \cline{2-4} 
        & Wednesday & 26\%         & 4                 \\ \cline{2-4} 
        & Thursday  & 25\%         & 4                 \\ \cline{2-4} 
        & Friday    & 25\%         & 4                 \\ \cline{2-4} 
        & Saturday  & 25\%         & 4                 \\ \cline{2-4} 
        & Sunday    & 25\%         & 4                 \\ \hline
    \end{tabular}
\end{table}

\subsection{Missing Pattern Adaptive Multiple Kernel Learning}
To work with missing channels, a straightforward way is to learn separate kernel weights for each missing pattern. However, there can be $\sum_{m=1}^{s}C_s^m$ missing patterns if we have $s$ channels, so it is possible that the data cannot cover every pattern. Besides, the data for one pattern can be less and contain only one type of label. Such a strategy also ignores the relationship between missing patterns. A likely choice would be to adjust the kernel weights according to missing patterns.

In order to allow adaptive kernel combination, we firstly modify the decision function for a sample $\mathbf{x}$ with $s$ channels into following form:
\begin{equation}
f(\mathbf{x})=\sum_{m=1}^s\eta_m(\mathbf{x})\left\langle\boldsymbol\omega_m,\phi_m(\mathbf{x}^{(m)})\right\rangle + b, \label{eq6}
\end{equation}
with $\left\langle \cdot,\cdot \right\rangle$ denotes the inner product of vectors and
\begin{equation}
\eta_m(\mathbf{x}) = p_m\mathbf{v}_m^\top \sum_{j=1}^{2s} p_j\mathbf{v}_j, \label{eq5}
\end{equation}
where $\mathbf{p}=[p_1,p_2,\cdots,p_{2s}]^\top$ is a binary vector generated by one-hot encoding on the missing pattern for sample $\mathbf{x}$. We introduce $V=[\mathbf{v}_1,\mathbf{v}_2,...,\mathbf{v}_{2s}]\in \mathbb{R}^{k\times2s}$ with latent dimension $k$ to represent embedding matrix for missing patterns. By Eq. (\ref{eq5}), we express the kernel weights as a second order polynomial mapping from missing patterns $\mathbf{p}$ with the coefficients given by related inner product of vectors in V. We give a simple example here to explain how we generate $\mathbf{p}$. Assume we have 3 data channels but for a sample the second one is missing, then: 
\begin{equation}
\mathbf{p} = [1,0,1,0,1,0]^\top.
\end{equation}
The first and third "1" mean we have first and third feature subsets for this sample. The fifth "1" serves as a complementary feature for missing channel 2. By doing so, the absence of a channel would make its kernel weight zero and influence the kernel weights of other presented channels. 

The motivation behind this is that we want to collect information from the missing pattern of each sample. Eq. (\ref{eq5}) also indicates that the kernel weight for a channel is decided by "seeing" the existence of other channels' data.

With similar notation to Eq. (\ref{eq2}), the optimisation problem after introducing adaptive kernel weight can be expressed as:
\begin{equation}
\begin{aligned} 
\min_{\{\boldsymbol\omega_m\}_{m=1}^s,b,\boldsymbol{\xi},V} \;& \frac12\sum_{m=1}^s\left\|\boldsymbol\omega_m\right\|_2^2 + C_1\sum_{i=1}^n\xi_i + C_2\left\|V\right\|_F^2\\ 
s.t.\quad& y_i\left(\sum_{m=1}^s\eta_m(\mathbf{x}_i)\boldsymbol\omega_m^\top \phi_m(\mathbf{x}_i^{(m)}) + b\right) \geq 1-\xi_i, \\
& \xi_i \geq 0,\quad i=1,2,...,n, \label{eq9}
\end{aligned} 
\end{equation}
where $C_1$ and $C_2$ are two regularisation parameters.  $\left\| \cdot \right\|_F^2$ denotes the Frobenius norm. We add a regularisation term for $V$ to prevent it from being arbitrary scaled up due to the norm constraint on $\boldsymbol\omega_m$.
\begin{theorem}
Adopting an adaptive kernel weight in Eq.(\ref{eq5}) would guarantee a positive semi-definite kernel for MKL.
\end{theorem}
\begin{proof}
For fixed $V$, one can obtain the dual form of Eq. (\ref{eq9}):
\begin{equation}
\max_{\boldsymbol\alpha\in\mathcal{Q}} \mathbf{1}^\top \boldsymbol\alpha-\frac12(\boldsymbol\alpha \circ \mathbf{y})^\top K_\eta (\boldsymbol\alpha \circ \mathbf{y}),
\end{equation}
where $\circ$ denotes element-wise product of vectors. $\mathbf{1}$ is a vector of all ones and $\mathcal{Q} = \left\{ \boldsymbol\alpha \in \mathbb{R} ^n : \boldsymbol\alpha ^ { \top } \mathbf { y } = 0,0 \leq \boldsymbol\alpha \leq C_1 \right\}$. $K_\eta$ is given by:
\begin{equation}
K_\eta = \sum_{m=1}^s \left(\left(\left( V^\top VP\right)\odot P\right)^\top\mathbb{I}_m\mathbb{I}_m^\top\left(\left(V^\top VP\right) \odot P\right)\right)\odot K_m, \label{eq10}
\end{equation}
where $\odot$ stands for the Hadamard product. $P=[\mathbf{p}_1,\mathbf{p}_2,\cdots,\mathbf{p}_n]$ with each column vector $\mathbf{p}_i\in\{0,1\}^{2s}$ denotes the missing pattern for sample $i$. $\mathbb{I}_m$ is a length-$2s$ indication vector with only $m$-th element 1. $\{K_m\}_{m=1}^s$ is the kernel matrix related to mapping $\{\phi_m(\cdot)\}_{m=1}^s$. Following Schur product theorem \cite{zhang2006schur}, $K_\eta$ is surely positive semi-definite. 
\end{proof}
Theorem 4.1 shows the correctness of our adaptive kernel weight in theory, but this problem is hard to solve in dual form because of the complicated form of $K_\eta$ in Eq. (\ref{eq10}).

\subsection{Sample Adaptive Multiple Kernel Learning}
If we train a unified model for all sets of railway points, we will possibly ignore some peculiarities of them even though we have included equipment details as features. Training separate models for each set of railway points performed even worse as we observed in initial experiments. These motivated us to modify our model so that it could be adjusted to fit each set of railway points. We revised the kernel weight in Eq.(\ref{eq5}) into the following format for a sample $\mathbf{x}$:
\begin{equation}
\eta_m(\mathbf{x}) = p_m\mathbf{v}_m^\top \sum_{j=1}^{2s} p_j\mathbf{v}_ja_j, \label{eq12}
\end{equation}
where we add a new vector $\mathbf{a}=[a_1,a_2,\cdots,a_{2s}]^\top$ to represent unique features of the set of railway points that generated sample $\mathbf{x}$.  

Related Eq. (\ref{eq12}) with Eq. (\ref{eq6}), we observe that the term $p_m$ could be omitted from Eq. (\ref{eq12}) if we set the mapping $\phi_m(\cdot)$ to a zero vector for missing channels. Thus we omit $p_m$ for simplicity of notation. If we have $T$ sets of railway points, then we will introduce $A=[\mathbf{a}_1,\mathbf{a}_2,\cdots,\mathbf{a}_T]\in\mathbb{R}^{2s\times T}$ with $T$ the total number of sets of railway points. Each column vector in $A$ stands for features of a set of railway points. Let $q(\mathbf{x}_i)$ be the mapping which maps $\mathbf{x}_i$ to index of railway points that generated $\mathbf{x}_i$. Eq. (\ref{eq12}) can be written into matrix form for sample $\mathbf{x}_i$:
\begin{equation}
\eta_m(\mathbf{x}_i) = \mathbb{I}_m^\top V^\top V \left( \mathbf{p}_i \circ \mathbf{a}_{q(\mathbf{x}_i)} \right), \label{eq13}
\end{equation} 

With $\eta_m(\mathbf{x}_i)$ given in Eq. (\ref{eq13}), corresponding optimisation problem becomes:
\begin{equation}
\begin{aligned} 
\min_{\{\boldsymbol\omega_m\}_{m=1}^s,b,\boldsymbol{\xi},V,A}  \; & \frac12\sum_{m=1}^s\left\|\boldsymbol\omega_m\right\|_2^2 + C_1\sum_{i=1}^n\xi_i+ C_2\left\|V\right\|_F^2\\ 
&+ C_3\left\|A - \mathbf{1}_{2s\times T}\right\|_F^2,\\ 
s.t.\quad& y_i\left(\sum_{m=1}^s\eta_m(\mathbf{x}_i)\boldsymbol\omega_m^\top \phi_m(\mathbf{x}_i^{(m)}) + b\right) \geq 1-\xi_i, \\
\;& \xi_i \geq 0,\quad i=1,2,...,n,
\end{aligned} \label{eq14}
\end{equation}
where $C_3$ is a regularisation parameter and $\mathbf{1}_{2s\times T}$ is a matrix of shape $2s\times T$ containing all ones. Notice that when A is a matrix of all ones, Eq. (\ref{eq12}) reduce to Eq.(\ref{eq5}). In other words, when $C_3$ is large enough, the two models would be equivalent. This regularisation term ensures an appropriate variance of models among different sets of railway points. One can also proof that such adaptive weights also retain a positive semi-definite kernel. 

\subsection{Optimisation}
As mentioned before, Eq.(\ref{eq9}) and Eq.(\ref{eq14}) are hard to optimise in dual form. What's more, we cannot fit such large data into memory if we pre-compute those 17 kernel matrices. Thanks to the random feature (RF) approximation \cite{rahimi2008random}, we can take an explicit form of mapped features hence avoiding calculation of the kernel matrices. This also facilitates the optimisation in the primal, which is much simpler. Given $\mathbf{x}\in\mathbb{R}^d$ and a predefined parameter $D$, the mapped features associated with a RBF kernel could be approximated by:
\begin{equation}
\phi(\mathbf{x}) {=} \sqrt{\frac{1}{D}}\left[\sin \left( \mathbf{g} _ { 1 } ^ { \top } \mathbf { x } \right) , \cos \left( \mathbf {g} _ { 1 } ^ { \top } \mathbf { x } \right) {, \cdots , }\sin \left( \mathbf{g} _ D ^ { \top } \mathbf { x } \right) , \cos \left( \mathbf{g} _ D ^ { \top } \mathbf { x } \right) \right]^\top,
\end{equation} 
where the entries of $G = [\mathbf{g}_1,\cdots,\mathbf{g}_D]\in\mathbb{R}^{d\times D} $ are drown i.i.d. from a Gaussian distribution $\mathcal{N}(0,\sigma^{-2})$ with $\sigma$ bandwidth of the RBF kernel. Many variants of RF approximation have been proposed in the literature. Here we implement the Fastfood \cite{le2013fastfood} for its simplicity and efficiency in memory usage.

\begin{algorithm}[t]
    \caption{Training Procedure by Mini-batch Gradient Descent}\label{a1}
    \begin{algorithmic}[1]
        \STATE 
        \textbf{Input:} Dataset $\mathcal{X}$ collected from $T$ sets of railway points. Latent dimension $k$ for $V$. Number of random features $\{d_m\}_{m=1}^s$ for each kernel. Hyper-parameters $C_1$, $C_2$, $C_3$. Learning rate $\beta$. Batch size $h$. The number of batches $H= \lfloor\frac{n}{h}\rfloor$.\\
        \textbf{Initialise:} $\{\boldsymbol\omega_m\}_{m=1}^s=\mathbf{0}$. $b=0$. $A=\mathbf{1}_{2s\times T}$. $V$ with values sampled from a uniform distribution $\mathcal{U}(0,1)$. 
        \FOR{$Epoch = 0$ to $M$}
            \STATE Shuffle the samples in $\mathcal{X}$ randomly.
            \STATE Split $\mathcal{X}$ into batches $X_1,X_2,\cdots,X_H$.
            \FOR {$i = 1,2,\cdots,H$}
                \STATE Get the index set $\mathcal{I}$ for support vectors in $X_i$  
                \STATE Update $V$ with step-size $\beta$ and sub-gradient in Eq. (\ref{gv})
                \STATE Update $A$ with step-size $\beta$ and sub-gradient in Eq. (\ref{ga})
                \STATE Update $b$ with step-size $\beta$ and sub-gradient in Eq. (\ref{gb})
                \STATE Update $\{\boldsymbol\omega_m\}_{m=1}^s$ with step-size $\beta$ and sub-gradient in Eq. \text{\qquad}(\ref{gw}).
            \ENDFOR
        \ENDFOR
    \end{algorithmic}
\end{algorithm}

Our optimisation problem can be rewritten into following form with hinge loss $L(x,y) = max(0,1-xy)$:
\begin{equation}
\begin{aligned}
\min \mathcal { L } = &\frac12 \sum_{m=1}^s\left\| \boldsymbol\omega_m \right\|_2 ^ { 2 }\\
&+C_1\sum_{i=1}^n L \left(y _i, \sum_{m=1}^s \eta_m \left(\mathbf{x}_i \right) \left\langle \boldsymbol\omega_m, \phi_m (\mathbf{x}_i^{(m)} ) \right\rangle + b \right)\\
&+ C_2\left\|V\right\|_F^2 + C_3\left\|A - \mathbf{1}_{2s\times T}\right\|_F^2, \\
\text{w.r.t.} &\quad \{\boldsymbol\omega_m\}_{m=1}^s,b,V,A,
\end{aligned}
\end{equation}
with $\eta_m(\mathbf{x}_i)$ defined in Eq.(\ref{eq13}), we can calculate the sub-gradients regarding these variables and get:
\begin{equation}
\frac{\partial\mathcal{L}}{\partial\boldsymbol\omega_m}=\boldsymbol\omega_m - C_1 \sum_{i\in\mathcal{I}} y_i \mathbb{I}_m^\top V^\top V\left( \mathbf{p}_i \circ \mathbf{a}_{q(\mathbf{x}_i)} \right) \phi_m(\mathbf{x}_i^{(m)}), \label{gw}
\end{equation}
\begin{equation}
\begin{aligned}
\frac{\partial\mathcal{L}}{\partial V}{=} &{-}C_1V \! \sum_{i\in\mathcal{I}}\!\sum_{m=1}^s y_i \boldsymbol\omega_m^\top \phi_m(\mathbf{x}_i^{(m)}) \! \left( \mathbb{I}_m \! \left( \mathbf{p}_i \! \circ \! \mathbf{a}_{q(\mathbf{x}_i)} \right)^\top \!+\! \left(\mathbf{p}_i \!\circ\! \mathbf{a}_{q(\mathbf{x}_i)} \right)\! \mathbb{I}_m^\top\right)\\
& + 2C_2 V,
\end{aligned}\label{gv}
\end{equation}
\begin{equation}
\begin{aligned}
\frac{\partial\mathcal{L}}{\partial\mathbf{a}_t} = &-C_1\sum_{i\in\mathcal{I}\cap\mathcal{T}_t} \sum_{m=1}^s \left(y_i \boldsymbol\omega_m^\top\phi_m (\mathbf{x}_i^{(m)} ) V^\top V \mathbb{I}_m \right) \circ \mathbf{p}_i\\
& + 2C_3(\mathbf{a}_t-\mathbf{1}_{2s}),
\end{aligned}\label{ga}
\end{equation}
\begin{equation}
\frac{\partial\mathcal{L}}{\partial b} = -C_1\sum_{i\in\mathcal{I}}y_i, \label{gb}
\end{equation}
where $\mathcal{I}=\{i | 1-y_if(\mathbf{x}_i)>0\}$ is the index set for support vectors. $\mathcal{T}_t=\{i | q(\mathbf{x}_i)=t \}$ is the index set of samples generated by railway points $t$.

With gradients calculated as Eq. (\ref{gw}) - Eq. (\ref{gb}), we adopted Mini-batch
gradient descent in optimisation. We trained the models for 50 epochs with a constant learning rate $\beta=0.0001$ and batch-size 256. Using $d_m$ to denote the dimension of random features for $m$-th kernel mapping, the computational complexity for calculating the gradients is $O(\sum_{m=1}^s d_mh + s^2k)$, which depends linearly on batch-size $h$ and can be computed efficiently. We summarise the training process in Algorithm \ref{a1}.

\begin{table}[b]
 \centering
 \caption{Dataset summary.}
 \label{dataset}
 \resizebox{0.47\textwidth}{7mm}{
     \begin{tabular}{@{}ccccc@{}}
      \toprule
      Dataset & \#instances & \#failures & \#railway points & \#incomplete instances \\ \midrule
      Points\_All & 58833 & 3900 & 350 & 25942 \\
      Points\_Subset & 905 & 183 & 5 & 98 \\ \bottomrule
     \end{tabular}
 }
\end{table}
\section{Experiments}
Our data were collected from 350 sets of railway points from 01/01/2014 to 30/06/2017, together with corresponding weather data downloaded from Australia Bureau of Meteorology. There are 58833 samples including 3900 failures. We named this dataset \textbf{Points-All}. We also built a subset consists of data from 5 most "vulnerable" sets of railway points, i.e. those with most failure samples, and named it \textbf{Points-Subset}. These datasets are imbalanced in label distribution. We have tried to weight the classes in training but saw no performance gains, so we did not adopt such strategy. Table \ref{dataset} summarises the statistics of our datasets.

\subsection{Baselines, Evaluation Metrics and Parameter Setting}
To show the effectiveness of our approach, we conducted experiments on the following methods.
\begin{itemize}
    \item MKL-ZF is the $l_p$-norm MKL method solved by the algorithm in \cite{kloft2011lp} with absent channels filled by zeros. We conducted experiments for $p$ ranges in $[10^0,10^1,10^2,10^3,10^4]$.
    \item MKL-MF is similar to MKL-ZF but with absent channels filled by the averages.
    \item MVL-MKL firstly imputes the missing values by the method in \cite{xu2015multi}, and then applied $l_p$-norm MKL with the imputed data. \cite{xu2015multi} is a competitive method for filling incomplete data similar to our case, so we included it in our baselines. 
    \item Absent Multiple Kernel Learning (AMKL) \cite{liu2015absent} is a state-of-the-art method for MKL with missing kernels. We only compared with AMKL on Points-Subset because it cannot be scaled up to fit our Points-All dataset.
    \item Single Source Classifiers (SSC) are the classifiers applied to single source data. For weather and movement logs data, there are still 7 data channels for each source. We use our method MAMKL as the classifier. For maintenance logs, equipment details and the data channel formed by concatenating all features, we filled the missing channels with means, and then used kernel SVM \cite{chang2011libsvm} for classification because these data sources only consist of one channel. 
    \item Missing Pattern Adaptive MKL (MAMKL) is the method proposed in this paper with kernel weights given by Eq. (\ref{eq5}). 
    \item Sample Adaptive MKL (SAMKL)  is the method proposed in this paper with kernel weights determined by Eq. (\ref{eq12}). 
\end{itemize}
For fair of comparison, for all methods, we used RF approximation for RBF kernels, and we fixed the random seed to make them determined. As such, $l_p$-norm MKL could also be applied to our Points-All dataset without pre-computed kernels. 

We used Area Under Receiver Operating Characteristic Curve (AUROC) and Area Under Precision Recall Curve (AUPRC) as our performance metrics for all the methods. For all non-convex methods, we repeated them 10 times to report the results with means and standard deviations. 
\begin{table}[t]
    \centering
    \caption{Experiment results on Points-Subset dataset. Best results are bold and the second best are underlined. We report the results with means and standard deviations (mean$\pm$std) for non-convex methods.}
    \label{subset}
    \begin{tabular}{|c|c|c|c|}
        \hline
        \multicolumn{2}{|c|}{Methods} & AUROC & AUPRC \\ \hline
        \multirow{5}{*}{MKL-ZF} &$p=10^0$& 0.737 & 0.436 \\ \cline{2-4} 
        &$p=10^1$& 0.921 & 0.791 \\ \cline{2-4} 
        &$p=10^2$& 0.902 & 0.784 \\ \cline{2-4} 
        &$p=10^3$& 0.920 & 0.789 \\ \cline{2-4} 
        &$p=10^4$& 0.921 & 0.790 \\ \hline
        \multirow{5}{*}{MKL-MF} &$p=10^0$& 0.646 & 0.289 \\ \cline{2-4} 
        &$p=10^1$& 0.923 & 0.800 \\ \cline{2-4} 
        &$p=10^2$& 0.887 & 0.770 \\ \cline{2-4} 
        &$p=10^3$& 0.887 & 0.767 \\ \cline{2-4} 
        &$p=10^4$& 0.906 & 0.780 \\ \hline
        \multirow{5}{*}{MVL-MKL} &$p=10^0$& 0.655$\pm$0.002 & 0.292$\pm$0.002 \\ \cline{2-4} 
        &$p=10^1$& 0.852$\pm$0.008 & 0.783$\pm$0.005 \\ \cline{2-4}
        &$p=10^2$& 0.898$\pm$0.010 & 0.788$\pm$0.015 \\ \cline{2-4}
        &$p=10^3$& 0.873$\pm$0.006 & 0.788$\pm$0.005 \\ \cline{2-4}
        &$p=10^4$& 0.873$\pm$0.006 & 0.788$\pm$0.004 \\ \hline
        \multirow{5}{*}{SSC} & Movement Logs & 0.663$\pm$0.001 & 0.380$\pm$0.001 \\ \cline{2-4} 
        & Weather & 0.864$\pm$0.035 & 0.781$\pm$0.036 \\ \cline{2-4} 
        & Maintenance Logs & 0.667 & 0.301 \\ \cline{2-4} 
        & Equipment Details & 0.516 & 0.217 \\ \cline{2-4} 
        & All Concatenated & 0.669 & 0.376 \\ \hline
        \multicolumn{2}{|c|}{AMKL} & 0.736 & 0.463 \\ \hline
        \multicolumn{2}{|c|}{MAMKL} & \underline{0.942$\pm$0.005} & \underline{0.831$\pm$0.016} \\ \hline
        \multicolumn{2}{|c|}{SAMKL} & \textbf{0.947$\pm$0.007} & \textbf{0.840$\pm$0.011} \\ \hline
    \end{tabular}
\end{table}
\begin{table}[t]
    \centering
    \caption{Experiment results on Points-All dataset. Best results are bold and the second best are underlined. We report the results with means and standard deviations (mean$\pm$std) for non-convex methods. }
    \label{alldata}
    \begin{tabular}{|c|c|c|c|}
        \hline
        \multicolumn{2}{|c|}{Methods} & AUROC & AUPRC \\ \hline
        \multirow{5}{*}{MKL-ZF} & $p=10^0$ & 0.699 & 0.218 \\ \cline{2-4} 
        & $p=10^1$ & 0.691 & 0.199 \\ \cline{2-4} 
        & $p=10^2$ & 0.696 & 0.205 \\ \cline{2-4} 
        & $p=10^3$ & 0.690 & 0.196 \\ \cline{2-4} 
        & $p=10^4$ & 0.692 & 0.197 \\ \hline
        \multirow{5}{*}{MKL-MF} & $p=10^0$ & 0.698 & 0.223 \\ \cline{2-4} 
        & $p=10^1$ & 0.684 & 0.204 \\ \cline{2-4} 
        & $p=10^2$ & 0.687 & 0.204 \\ \cline{2-4} 
        & $p=10^3$ & 0.682 & 0.198 \\ \cline{2-4} 
        & $p=10^4$ & 0.668 & 0.176 \\ \hline
        \multirow{5}{*}{MVL-MKL}&$p=10^0$& 0.678$\pm$0.001 & 0.168$\pm$0.002\\ \cline{2-4} 
        &$p=10^1$& 0.671$\pm$0.001 & 0.159$\pm$0.001 \\ \cline{2-4} 
        &$p=10^2$& 0.670$\pm$0.001 & 0.159$\pm$0.001 \\ \cline{2-4} 
        &$p=10^3$& 0.672$\pm$0.002 & 0.158$\pm$0.001 \\ \cline{2-4} 
        &$p=10^4$& 0.674$\pm$0.002 & 0.159$\pm$0.003 \\ \hline
        \multirow{5}{*}{SSC} & Movement Logs & 0.546$\pm$0.010 & 0.093$\pm$0.001 \\ \cline{2-4} 
        & Weather & 0.677$\pm$0.003 & 0.197$\pm$0.008 \\ \cline{2-4} 
        & Maintenance Logs & 0.567 & 0.098 \\ \cline{2-4} 
        & Equipment Details & 0.517 & 0.085 \\ \cline{2-4} 
        & All Concatenated & 0.622 & 0.133 \\ \hline
        \multicolumn{2}{|c|}{MAMKL} &\underline{0.721$\pm$0.002}&\underline{0.261$\pm$0.009}\\ \hline
        \multicolumn{2}{|c|}{SAMKL} &\textbf{0.734$\pm$0.002}&\textbf{0.270$\pm$0.002}\\ \hline
    \end{tabular}
\end{table}
For the Points-All dataset, we split it into 60\% training data, 20\% validation data and 20\% test data. The linear kernel was used for the data channels from equipment details and maintenance logs. We set same bandwidth for RBF kernels on 7 data channels from weather data. The bandwidth is chosen from $[\sigma^{-2},\sigma^{-1},\sigma^0,\sigma^1,\sigma^2]$ according to the AUROC on validation data using SVM with sum of these 7 kernels as input. $\sigma$ is the standard deviation of weather data. The same criterion was adopted to select the parameter of RBF kernels for 7 data channels from movement logs and 1 data channel from concatenated features. The dimensions of RFs for approximating RBF kernels were set to 1024, 2048 and 2048 for movement logs, weather and concatenated features respectively. All other parameters were chosen from some appropriately large ranges based on the AUROC of related methods on validation data. For Points-Subset, we randomly selected 80\% data as training set and the remaining 20\% as the test set. Parameters for them were decided by 5-fold cross-validation on the training set. 

\subsection{Results on Points-Subset Dataset}
Table \ref{subset} shows the experiment results on Points-Subset dataset. $l_p$-norm MKL got inferior results when $p=1$, for the reason that they tended to find a sparse combination of kernels. This means our data channels carry the complementary information, so only use some of them could not produce a good result. Experiment results on SSC verify our argument that only use data from one source is not enough. 
The prefilling method did not perform best, because filling the missing data in advance and used them in training will possibly introduce another source of error. Although AMKL appropriately takes into account the missing pattern in trainings, it keeps a fixed kernel weight in testing. Besides, it is designed for $l_1$-norm MKL, so it did not perform well in our experiments. 
It is clear that our method outperforms other baselines in terms of both AUROC and AUPRC. We attribute the improvement to the combination of multi-source data and the sample adaptive kernel weights.

\subsection{Results on Points-All Dataset}
Table \ref{alldata} shows the experiment results on Points-All dataset. By training on all data, we also included some sets of railway points with only a few failure cases. The proportion of incomplete samples is also higher than that in Points-Subset. These added up to our difficulties in predicting the failures. As in Table \ref{alldata}, results with $p=1$ is often better. This means traditional MKL cannot fully exploit the merits of multiple kernels. Our method still can beat other baselines on both AUROC and AUPRC, and see improvement compared to SSC. Notice that SAMKL is much better than MAMKL in this dataset, which verifies the effectiveness of sample adaptive kernel weight. This could guarantee a reliable warning for failures predicted by our model.  

For each set of railway points, the number of samples is usually less than 180. Only several failures are observed for some points. We also trained many classifiers each for one set of railway points, but the results were unsatisfactory, so we did not list them here. 

\section{Conclusion}
We have designed a novel approach for combining incomplete multi-source data to predict the failure of railway points. It was developed based on the multiple kernel learning framework but went a step further by exploiting the missing patterns and sample-specific features. With the involvement of domain experts, we grouped our data weekly and split each week into a daily format to form 17 data channels and built 17 kernels. In this format, we can express the missing patterns of samples clearly. After that, we put forward a missing pattern adaptive MKL to leverage the information carried by missing patterns. We also considered the distinct properties of each set of railway points, and further improved the prediction results by our SAMKL algorithm. Experiments show that our model can output reliable warnings for railway points, and can predict the failures precisely for those frequently failed points.

In the future, we are going to apply more kernel functions on a single data channel, and reduce the resulting extra optimisation time by parallel computing through GPU.     

\begin{acks}
The authors greatly appreciate the financial support from the Rail Manufacturing Cooperative Research Centre (funded jointly by participating rail organisations and the Australian Federal Government's Business Cooperative Research Centres Program) through Project R3.7.2 - Big data analytics for condition-based monitoring and maintenance.
\end{acks}

\bibliographystyle{ACM-Reference-Format}
\bibliography{sample-base}

\end{document}